\documentclass[11pt]{article}
\usepackage{fullpage}
\usepackage{subfig}
\usepackage[subfigure]{tocloft}
\usepackage[usenames,dvipsnames]{xcolor}

\usepackage[linktocpage=true,
	pagebackref=true,colorlinks,
	linkcolor=BrickRed,citecolor=blue,
	bookmarks,bookmarksopen,bookmarksnumbered]
	{hyperref}

	\usepackage{latexsym,graphicx,epsfig,color}
\usepackage{amsfonts,amssymb,amsmath,amsthm,amstext}
\usepackage{lmodern}
\usepackage{enumitem}
\setitemize{itemsep=2pt,topsep=0pt,parsep=0pt}
\usepackage{url,setspace}
\usepackage{multirow}
\usepackage{rotating}
\usepackage{makeidx}
\usepackage{tikz}
\usepackage{accents}
\usepackage{xspace}
\usepackage{algorithm,algpseudocode}
\usepackage{bm}
\usepackage{changepage} 	
\usepackage{thmtools,thm-restate}

\allowdisplaybreaks

\usetikzlibrary{decorations.markings}
\usetikzlibrary{arrows}
\tikzstyle{block}=[draw opacity=0.7,line width=1.4cm]
\tikzstyle{graphnode}=[circle, draw, fill=black!20, inner sep=0pt, minimum width=6pt]
\tikzstyle{point}=[circle, draw, fill=black!30, inner sep=0pt, minimum width=1pt]
\tikzstyle{input}=[rectangle, draw, fill=black!75,inner sep=3pt, inner ysep=3pt, minimum width=4pt]
\tikzstyle{unmatched}=[graphnode,fill=black!0]
\tikzstyle{shaded}=[graphnode,fill=black!20]
\tikzstyle{matched}=[graphnode,fill=black!100]  	
\tikzstyle{matching} = [ultra thick]
\tikzset{
    >=stealth',
    pil/.style={
           ->,
           thick,
           shorten <=2pt,
           shorten >=2pt,}
}
\tikzset{->-/.style={decoration={
  markings,
  mark=at position .5 with {\arrow{>}}},postaction={decorate}}}

\makeindex

\newtheorem{theorem}{Theorem}[section]
\newtheorem{claim}[theorem]{Claim}
\newtheorem{proposition}[theorem]{Proposition}
\newtheorem{lemma}[theorem]{Lemma}

\newtheorem{fact}[theorem]{Fact}
\theoremstyle{definition}

\usepackage{thmtools,thm-restate} 

\def\eps {\epsilon}

\def \reals {\mathbb{R}}

\newcounter{note}[section]
\newcommand{\snote}[1]{\refstepcounter{note}$\ll${\bf Sahil~\thenote:}
  {\sf \color{red}  #1}$\gg$\marginpar{\tiny\bf SS~\thenote}}

\newcommand{\BwK}{\ensuremath{\textsc{BwK}}\xspace}
\newcommand{\BwKp}{\ensuremath{\textsc{BwK}_p}\xspace}
\newcommand{\OLVC}{\ensuremath{\textsc{OLVC}}\xspace}
\newcommand{\OLVCp}{\ensuremath{\textsc{OLVC}_p}\xspace}

\newcommand{\OGLBp}{\ensuremath{\textsc{OGLB}_p}\xspace}
\newcommand{\Regret}{\textsc{Regret}}
\newcommand{\calA}{\ensuremath{\mathcal{A}}}
\newcommand{\load}{\mathbf{\Lambda}}
\newcommand{\OPT}{\textsc{OPT}\xspace}
\newcommand{\OPTOLVC}{\ensuremath{\textsc{OPT}_{\textsc{OLVC}}}\xspace}
\newcommand{\OPTBwK}{\ensuremath{\textsc{OPT}_{\textsc{BwK}}}\xspace}

\newcommand{\ALGOLVC}{\ensuremath{\textsc{ALG}_{\textsc{OLVC}}}\xspace}
\newcommand{\ALGBwK}{\ensuremath{\textsc{ALG}_{\textsc{BwK}}}\xspace}
\newcommand{\E}{\mathbb{E}}
\newcommand{\one}{\mathbf{1}\xspace}
\newcommand{\zero}{\mathbf{0}\xspace}
\def\eps {\epsilon}
\newcommand{\IGNORE}[1]{}
\def \reals {\mathbb{R}}

\usepackage{mdframed}

\newtheorem{mdresult}{Main Result}
\newenvironment{mainresult}{\begin{mdframed}[backgroundcolor=lightgray!40,topline=false,rightline=false,leftline=false,bottomline=false,innertopmargin=5pt]\begin{mdresult}}{\end{mdresult}\end{mdframed}}

\title{Online Learning with Vector Costs\\
and  Bandits with Knapsacks}

\author{%
Thomas Kesselheim\thanks{(thomas.kesselheim@uni-bonn.de)
 Institute of Computer Science,	    University of Bonn.}
 \and
 {Sahil Singla}\thanks{(singla@cs.princeton.edu)
    Department of Computer Science at
        Princeton University and
        School of Mathematics at 
        Institute for Advanced Study.        Supported in part by the Schmidt Foundation.}
}

\begin{document}

\maketitle

\begin{abstract}{We introduce online learning with vector costs (\OLVCp)  where in each time step $t \in \{1,\ldots, T\}$, we need to play an action $i \in \{1,\ldots,n\}$ that incurs an unknown vector cost in $[0,1]^{d}$. The goal of the online algorithm is to minimize the $\ell_p$ norm of the sum of its cost vectors. This captures the classical online learning setting for $d=1$, and is interesting for general $d$ because of applications like online scheduling where we want to balance the load between different machines (dimensions). 

We study \OLVCp in both stochastic and adversarial arrival settings, and give a general procedure to reduce the  problem from $d$ dimensions to  a single dimension. This allows us to use classical online learning algorithms in both full and bandit feedback models to obtain (near) optimal results. In particular, we obtain a single algorithm   (up to the choice of learning rate) that gives sublinear regret for stochastic arrivals and a tight $O(\min\{p, \log d\})$ competitive ratio for adversarial arrivals.

The \OLVCp problem also occurs as a natural subproblem when trying to solve the popular Bandits with Knapsacks (\BwK) problem. This connection allows us to use our \OLVCp techniques to obtain (near) optimal results for \BwK in both stochastic and adversarial settings. In particular, we obtain a tight $O(\log d \cdot \log T)$ competitive ratio algorithm for adversarial \BwK, which improves over the $O(d \cdot \log T)$ competitive ratio algorithm  of Immorlica et al. [FOCS'19].
}\end{abstract}

\setcounter{tocdepth}{1}

{\small
\begin{spacing}{0.01}
   \tableofcontents
\end{spacing}
}

\clearpage


\section{Introduction}\label{sec:intro}

The field of online learning has seen tremendous progress in the last three decades since the seminal work of~\cite{LittlestoneW-FOCS89}. This is primarily because it is the simplest model to study exploration vs.\ exploitation trade-off while still being powerful enough to model several applications like clinical trials, ad-allocation, and portfolio optimization.
The classic setup of online learning is that we are given a set $\calA$ of $n$ actions and in each time step $t\in [T]$ we take an action  $x^{(t)} \in \Delta_n$\footnote{We use the notation $[T]:= \{1,2,\ldots,T\}$ and  write $\Delta_n$ to mean the simplex $\{\mathbf{x} \in [0,1]^n  \mid \sum_{i=1}^n x_i = 1 \}$.}. After taking the action we  incur a \emph{scalar} cost  $  c^{(t)} x^{(t)} $, where $c^{(t)} \in [0,1]^{1 \times n}$ is a row vector of costs, and receive a \emph{feedback} about $c^{(t)}$.
The goal of the algorithm is to minimize its total cost compared to that of the best \emph{fixed action}  $x^\ast \in \Delta_n$ chosen in hindsight. Several variants of this basic  problem have been studied in the literature, depending on the amount of feedback that the algorithm receives (full vs. partial) and who draws the costs $c^{(t)}$ (adversary vs. stochastic)~(see books~\cite{CL-Book06,BubeckC-Book12,Hazan-Book16,Slivkins-Book19,LS-Book20}).

In this work we study online learning problems where  the actions incur \emph{vector costs} in $[0,1]^d$. The algorithm wants to minimize the $\ell_p$ norm, given some integer $p \geq 1$, of the total cost vector. Formally, in the \emph{online learning with vector costs} (\OLVCp) problem
there is a set $\calA$ of $n$ actions and in each  time step $t$
we take an action  $x^{(t)} \in \Delta_n$. After taking the action we  incur a \emph{vector} cost $ C^{(t)} x^{(t)} $, where $C^{(t)} \in [0,1]^{d \times n}$ is a cost matrix, and receive a feedback about $C^{(t)}$. The goal of the algorithm is to minimize  the $\ell_p$ norms of its total cost vector compared to that of a fixed benchmark distribution $x^\ast \in \Delta_n$ over the actions\footnote{The optimal fixed $x^\ast$ is a distribution over actions, rather than being a fixed deterministic action as in the $d=1$ case. E.g., suppose $d=2$ and $p=\infty$ and there are two actions with the cost matrix $C^{(t)}$ always being a 2-dimensional identity matrix. Here the optimal distribution is $(0.5,0.5)$, which is a constant factor better than any fixed action.}.
This problem clearly captures classic online learning for $d=1$. Again we want to  study the problem depending on both the amount of feedback and on who draws the cost matrices.

The motivation to study online learning with vector costs is twofold. We first argue below that it is a natural problem in itself because it models the classical scheduling problem of online load balancing. Next we argue that the problem is also interesting because it is a natural step towards designing (optimal) algorithms for the Bandits with Knapsacks (\BwK) problem, both in the stochastic and adversarial settings. The \BwK problem has been  applications like dynamic item pricing, repeated auctions, and dynamic procurement (e.g., see \cite{BKS-JACM18,AD-EC14,SankararamanS-AISTATS18,ISSS-FOCS19}).


\subsection{Online Load Balancing}
In the classical \emph{online generalized load balancing} (\OGLBp) problem, a sequence of $T$ jobs arrive one-by-one and the algorithm has to process them on $d$ machines to minimize the $\ell_p$ norm of the total load vector. Each job can be processed in $n$ ways, and playing $x^{(t)} \in \Delta_n$ incurs a vector load $C^{(t)}x^{(t)}$, where $C^{(t)} \in [0,1]^{d \times n}$ is a load matrix.
It generalizes the fundamental makespan minimization problem where we want to minimize the maximum load, i.e., the $\ell_{\infty}$ norm. The motivation to study $\ell_{p}$ norms for arbitrary $p$ comes from noticing that $p=1$ captures the total load, $p=\infty$  captures the maximum load, and $p$ between $1$ and $\infty$  interpolates between the two extremes. 
The problem and its special cases have been well studied both for adversarial loads (\cite{AwerbuchAGKKV-FOCS95,Caragiannis-SODA08,ImKKP-SICOMP19}) and stochastic loads (\cite{Molinaro-SODA17}).

The major difference between the {classical} \OGLBp and our \OLVCp problems is that in the former the algorithm  gets to see the load/cost matrix $C^{(t)}$ \emph{before} playing the action, whereas in the latter the algorithm receives a full/partial feedback only \emph{after} playing the action. Depending on the application in hand, both models make sense as it may or may not be possible to find the cost matrix before playing the action.
Of course, this makes a huge difference in what guarantees one can hope to achieve. Indeed, in \OGLBp the benchmark is an arbitrary sequence  of actions (policy) chosen in hindsight, while in \OLVCp the benchmark is a fixed distribution $x^\ast$ of actions. Given a benchmark, we can compare different algorithms using  the \emph{regret} or the \emph{competitive-ratio} framework (defined below), depending on whether the ratio of the costs of the algorithm and the benchmark tend towards $1$, or not, as $T\rightarrow \infty$.

In \OLVCp the benchmark is a fixed distribution  $x^\ast \in \Delta_{n}$  on actions. Its cost  is defined to be 
\begin{align} \label{eq:defnOPT} 
 \OPTOLVC :=  \Big\|\sum_{t = 1}^T C^{(t)} \cdot x^\ast \Big\|_p \quad \text{or} \quad \OPTOLVC :=  \E \Big[ \Big\|\sum_{t = 1}^T C^{(t)} \cdot x^\ast \Big\|_p \Big],
\end{align}
depending on  the adversarial or the stochastic setting, respectively. The total cost of the algorithm is 
$\ALGOLVC:= \E \big[ \big\|\sum_{t = 1}^T C^{(t)} \cdot x^{(t)} \big\|_p \big]$,
where the expectation is taken over any randomization of the algorithm and of the cost matrices. 
We say an algorithm  is $\alpha$-\emph{competitive} if $\ALGOLVC \leq \alpha \cdot \OPTOLVC + o(T)$. For $1$-competitive algorithms we define \emph{regret} to be $\ALGOLVC - \OPTOLVC$.

The following is our main result for \OLVCp, which is formally stated in Theorems~\ref{prop:stochLoad} and~\ref{prop:advLoad}.

\begin{mainresult} \label{thm:OLVCp}
\emph{For stochastic \OLVCp there exists an algorithm that guarantees $\ALGOLVC \leq \OPTOLVC + o(T)$. For adversarial \OLVCp there exists an algorithm that guarantees $\ALGOLVC \leq O(\min\{p, \log d\}) \cdot \OPTOLVC + o(T)$.}
\end{mainresult}
So, for $\ell_\infty$ norm, we get an $O(\log d)$-competitive algorithm for \OLVC with adversarial arrivals.
A  powerful feature of this result is that the same algorithm (up to the choice of the learning-rate parameter) achieves (near) optimal guarantees in both the stochastic and adversarial settings.

This work is not the first one to consider online learning problems with vector costs. Indeed, \cite{AD-EC14} study the \emph{Bandits with Convex Knapsacks} problem where one also incurs vector costs, and the regret is measured in terms of the distance from a given convex body. Since $\ell_p$ norm is convex, their UCB-style algorithms can be used to obtain sublinear regrets for our \emph{stochastic} \OLVCp.
Our challenges and techniques, however, are very different from \cite{AD-EC14} for adversarial arrivals. On the one hand, UCB-style techniques do not apply for adversarial arrivals. On the other hand, we show in \S\ref{sec:lowerBounds} that every online algorithm is $\Omega(\min\{p, \log d\})$-competitive for adversarial \OLVCp. The fact that our adversarial \OLVCp algorithm also gives  near-optimal guarantees for  stochastic \OLVCp is a bonus. 
Another relevant paper is~\cite{AFFT-ITCS14}, which also considers online learning for adversarial cost vectors but  defines a new performance measure based on ``excuse blocks'' to obtain sublinear regrets. A related but different model for online load balancing has also been studied in~\cite{EDKMM-COLT09,RST-COLT11,LHT=arXiv20}.



\subsection{Bandits with Knapsacks} \label{sec:introBwK}

The Bandits with Knapsacks (\BwK) problem was introduced in~\cite{BKS-JACM18} to handle  applications of multi-armed bandits with resource constraints. It is very similar to the above \OLVC problem, but is in a maximization setting. We define its generalization to arbitrary $\ell_p$ norms, denoted by \BwKp as follows (for $p=\infty$ this is the same as \BwK). We are given a budget $B \geq 0$ and a set $\calA$ of $n$ actions. In each  time step $t \in [T]$, we take an action  $x^{(t)} \in \Delta_n$.  After taking the action, we receive a \emph{scalar} reward $r^{(t)} x^{(t)}$ and incur a \emph{vector} cost $ C^{(t)} x^{(t)} $,  where $r^{(t)} \in [0,1]^{1 \times n}$ is a row vector of rewards and $C^{(t)} \in [0,1]^{d \times n}$ is a cost matrix. Moreover, we also receive a  partial/full feedback about $C^{(t)}$ and $r^{(t)}$.  We assume there is a \emph{null action} in $\calA$ that gives $0$ reward and $\zero_d$ vector cost\footnote{We use the notation $\zero_d$ (or $\one_d$) to mean a $d$-dimensional vector of all $0$s (or $1$s).}, which allows us to skip some time steps. Our goal is to maximize the total  reward received while  ensuring that the $\ell_p$ norm of the total cost vector is  less than $B$. After exhausting the budget, we are only allowed to choose the null action (thus we obtain no further reward).

The benchmark for \BwKp is any fixed distribution $x^\ast \in \Delta_n$ over actions. Let $\tau^\ast$ be the time step at which the benchmark exhausts its budget, i.e., $\tau^\ast \leq T$ is maximal such that $\| \sum_{t=1}^{\tau^\ast} C^{(t)} x^\ast \|_p \leq B$ or $\| \E[\sum_{t=1}^{\tau^\ast} C^{(t)} x^\ast] \|_p \leq B$ depending on adversarial/stochastic setting. Now the optimal reward  $\OPTBwK$  is defined to be $\sum_{t=1}^{\tau^\ast} r^{(t)} x^\ast$ or $\E[\sum_{t=1}^{\tau^\ast} r^{(t)} x^\ast]$, respectively, and the value of the algorithm $\ALGBwK:= \E \Big[ \sum_{t = 1}^T r^{(t)} x^{(t)} \Big]$, where the expectation is taken over any randomization of the algorithm and the cost matrices.
In such a maximization setting, we say an algorithm is $\alpha$-\emph{competitive} if $ \ALGOLVC \geq \frac{1}{\alpha} \cdot \OPTOLVC - o(T)$. For $1$-competitive algorithms, we define its \emph{regret} to be $\OPTBwK - \ALGBwK$. That is, also for \BwKp, the competitive ratio is always at least $1$ and bounds on the regret will be non-negative.

The \BwK problem has been generalized in various directions since the original work of ~\cite{BKS-JACM18} (see related work in \S\ref{sec:relatedWork}). All these works were, however, only in the stochastic setting.  In a very recent work, \cite{ISSS-FOCS19} initiated the study of \BwK in an adversarial setting where both the costs and the rewards are chosen by an adversary. Since sub-linear regrets are impossible here, they  design algorithms with a small competitive ratio $\alpha$. In particular, assuming the value of $\OPTBwK$ is known, they give an $O(d)$-competitive algorithm and show that every online algorithm is $\Omega(1)$-competitive. Without this assumption, they give an $O(d \cdot \log T)$-competitive algorithm and show that every online algorithm is $\Omega(\log T)$-competitive.

Although \cite{ISSS-FOCS19} give the optimal dependency on the time-horizon $T$ in the competitive ratio of adversarial \BwK, they leave open what is the optimal dependency on the number of dimensions $d$. Understanding the optimal dependency on $d$ is useful for applications with large $d$, e.g., in online ad-allocations every bidder has a budget constraint and $d$ corresponds to the number of bidders~\cite{Mehta-TCS12}.  Our second main result (formally proven in Theorems~\ref{prop:BwKAdversarial} and~\ref{prop:BwKStochastic}) resolves this open question of \cite{ISSS-FOCS19} and gives the optimal competitive ratio for adversarial $\BwK$, and its generalization to adversarial $\BwKp$. 

\begin{mainresult} \label{thm:BwKp}
\emph{For stochastic \BwKp there exists an algorithm that guarantees $\ALGBwK \geq \OPTBwK - o(T)$. For adversarial \BwKp there exists an algorithm that guarantees $\ALGBwK \geq \frac{1}{O(\min\{p, \log d\})} \cdot  \OPTBwK - o(T)$ assuming it is given the value $\OPTBwK$, and \linebreak $  \ALGBwK \geq \frac{1}{O(\min\{p, \log d\} \log T)} \cdot\OPTBwK - o(T)$ when $\OPTBwK$ is not known.}
\end{mainresult}


Note that an immediate corollary of the result is an $O(\log d)$-competitive algorithm for adversarial \BwK when \OPTBwK is known, and an $O(\log d \cdot \log T)$-competitive algorithm otherwise. This improves on the $O(d)$-competitive and the $O(d \cdot \log T)$-competitive algorithms  of \cite{ISSS-FOCS19} in the respective settings. 
Our proof crucially relies on the connections between \BwKp and \OLVCp, after we Lagrangify the budget constraint of \BwKp. In \S\ref{sec:lowerBounds} we also show that our competitive ratios for adversarial arrivals  in Main Result~\ref{thm:BwKp} are optimal.


\subsection{Further Related Work} \label{sec:relatedWork}
Since the field of online learning is vast, we only discuss the  most relevant work and refer the reader to the books mentioned in the introduction for other references. At a high-level, our proofs go by reducing the $d$-dimensional problem to a surrogate classical $1$-dimensional online learning problem against an \emph{adaptive} adversary. 
Now we  run classical algorithms like Hedge~\cite{FS-GEB99} in the full feedback setting and EXP-3.P~\cite{AuerCFS-SICOMP02} in the bandit feedback setting. 

Since the original work of \cite{BKS-JACM18} (conference version appeared in 2013),  the \BwK problem has been generalized in various directions like  concave-rewards with convex-costs~\cite{AD-EC14}, contextual bandits~\cite{AgrawalD-NIPS16,AgrawalDL-COLT16}, and combinatorial semi-bandits~\cite{SankararamanS-AISTATS18}. All these works are in the stochastic setting and use UCB-style algorithms to obtain sub-linear regrets.  \cite{ISSS-FOCS19} introduce the adversarial \BwK problem and
show that their techniques  also apply to some of these directions for  adversarial arrivals.  
Another relevant work is that of \cite{RangiFT-IJCAI19} who consider adversarial \BwK but only for $d=1$.

The competitive analysis framework has been popular in the online algorithms community since at least the work of~\cite{SleatorTarjan-85}. We refer the readers to the books~\cite{borodin2005online,BuchbinderNaor-Book09}. Unlike online learning,  these results are in a setting where the algorithm gets full-information about the next arrival before making its decision. In recent years, competitive analysis has also become popular for online learning. In particular, \cite{KakadeKL-SICOMP09} showed that for any linear combinatorial optimization problem for which we know an $\alpha$-approximation offline algorithm, we can also obtain an $\alpha$-competitive online combinatorial optimization algorithm against the benchmark that always plays a fixed feasible solution chosen in hindsight. Their result does not apply to our setting as our objective is not additive over time steps.




Finally, we remark that vector losses also appear in the classical Blackwell's approachability theorem~\cite{Blackwell-Journal56}. Due to the work of \cite{AbernethyBH-COLT11}, it is known this is equivalent to no-regret learning. \cite{AD-EC14} exploit this connection to design a fast algorithm for stochastic \BwK. These techniques, however, seem less relevant for our adversarial arrival problems where sublinear regrets are impossible.



\section{Our Approach via a 1-d Surrogate Online Learning Problem}
\label{sec:surrogateGame}

The idea of our algorithms is to reduce the multi-dimensional problem with vector costs to a single dimensional problem. In each step $t$, each action $i$ is assigned a scalar loss $c^{(t)}_i \in [0, \| \one \|_p ]$. This loss will mirror the increase of the $\ell_p$ norm that would be incurred by choosing action $i$ in step $t$. In this surrogate problem, standard online learning algorithms can be applied.

In more detail, for any time $t\in \{0, \ldots, T\}$ and any action sequence $x^{(1)}, \ldots, x^{(t-1)}$  of the algorithm, let $\load^{(t)} := \sum_{s = 1}^t C^{(s)} \cdot x^{(s)}$ denote the load vector after step $t$. That is, in the load balancing setting, our goal is to minimize $\lVert \load^{(T)} \rVert_p$.

To define the cost of an action, we use a smooth approximation $\Psi (\cdot)$ of the $\ell_p$ norm, which we introduce below. We set the cost of action/expert $i\in [n]$ in step $t$ in the surrogate problem to be 
\[  c_i^{(t)} := \langle C^{(t)} e_i , \nabla \Psi (\load^{(t-1)}) \rangle,\]
 where $e_i$ is the $i$-th unit vector. The idea is to choose $\Psi$ so that $\sum_{t = 1}^T \langle C^{(t)} x^{(t)}, \nabla \Psi (\load^{(t-1)}) \rangle \approx \lVert \load^{(T)} \rVert_p$.
Note that $c^{(t)}$ depends on $x^{(1)}, \ldots, x^{(t-1)}$ via $\load^{(t-1)}$. So the surrogate problem is against  an \emph{adaptive} (non-oblivious) adversary.

\paragraph{A Smooth Approximation of the Norm.}
Let us first see an example why smoothening the $\ell_p$ norm is important. Consider $p=\infty$ and  there are only two actions: the first action puts $(1-\epsilon)$ load on every dimension for $\epsilon \rightarrow 0$ and the second action puts $1$ load on one of the $d$ dimensions chosen at random. Here a greedy algorithm w.r.t. the $\ell_\infty$ norm will chose the first action (to save an $\epsilon$), but we should have chosen the second action. Hence the greedy algorithm is $\Omega(d)$-competitive.

We use a smooth approximation of $\ell_p$ norm due to~\cite{Molinaro-SODA17}. 
Fixing $\epsilon > 0$, define for any   $\load \in \reals^d_{\geq 0}$,
\[ \Psi(\load) ~~:=~~ \frac{p}{\epsilon} \Big( \Big\| \one + \frac{\epsilon \load}{p} \Big\|_p - 1 \Big) ~~=~~ \Big\| \frac{p}{\epsilon} \one + \load \Big\|_p - \frac{p}{\epsilon}.\] 
 In particular,  for $p \to \infty$, we have $\Psi(\load) \to \frac{1}{\epsilon} \ln \big( \sum_j \exp(\epsilon \Lambda_j) \big)$, which is a common smooth approximation of the $\ell_\infty$ norm. Sometimes we abbreviate $\Psi^{(t)} := \Psi(\load^{(t)})$. 

The following fact follows from the triangle inequality.

\begin{fact}(Additive Approximation~\cite{Molinaro-SODA17}) \label{fact:PsiToLoad}
For any integer $p\geq 1$ and load $\load \in \reals^d_{\geq 0}$, we have $\| \load  \|_p \leq  \Psi(\load)  \leq  \| \load \|_p + \frac{p}{\epsilon} \cdot  \big(\|\one\|_p -1 \big)$.
\end{fact}

\noindent Note that $p \big(\|\one\|_p -1 \big) \leq \min\{p, \ln d\}$ for all $p$. So, this is a meaningful bound even for $p \to \infty$.

The gradient of $\Psi$ is bounded like that of the function $\| \cdot\|_p$ (see Fact~\ref{fact:PsiNorm}). The advantage of $\Psi$ over $\| \cdot\|_p$ is that  changes in its gradient are also bounded (see Fact~\ref{fact:PsiSmoothness}).

\begin{fact}\label{fact:PsiNorm} (Gradient Norm \cite{Molinaro-SODA17}) For any integer $p \geq 1$,  $\|\nabla \Psi\|_q \leq 1$ for $q = (1 - \frac{1}{p})^{-1}$.
\end{fact}

\begin{fact} (Gradient Stability \cite{Molinaro-SODA17}) \label{fact:PsiSmoothness}
For any integer $p\geq 1$, load $\load \in \reals^d_{\geq 0}$ and load increase $\textbf{z} \in [0,1]^d$,  we have coordinate-wise $\nabla \Psi(\load + \textbf{z}) \leq (1 + \epsilon) \cdot \nabla \Psi(\load)$.
\end{fact}

There are several other ways known to get similar smooth approximations, see \cite{Nesterov-MP05}. For our purposes it is  convenient that the gradient stability gives a multiplicative approximation.

\paragraph{Implications for the Surrogate Problem.}

Recall that we set the cost of action/expert $i\in [n]$ in step $t$ to be $c_i^{(t)} = \langle C^{(t)} e_i , \nabla \Psi (\load^{(t-1)}) \rangle$.
As an immediate consequence of gradient stability, we can bound the increase of $\Psi$ by the cost in the surrogate problem as follows.
\begin{claim}  \label{claim:PsiSmoothness}
In any step $t\in[T]$, the change in $\Psi$ is approximated by the cost in the surrogate problem as $\Psi(\load^{(t)}) -  \Psi(\load^{(t-1)})  \leq   (1 + \epsilon) \cdot \langle C^{(t)} x^{(t)}, \nabla \Psi (\load^{(t-1)}) \rangle$.
\end{claim}
\begin{proof}
Applying convexity, we get $\Psi(\load^{(t-1)}) \geq \Psi(\load^{(t)}) + \langle \load^{(t-1)} - \load^{(t)}, \nabla \Psi (\load^{(t)}) \rangle$. Note that $\load^{(t-1)} - \load^{(t)} = - C^{(t)} x^{(t)}$, so $\Psi(\load^{(t)}) -  \Psi(\load^{(t-1)}) \leq \langle C^{(t)} x^{(t)}, \nabla \Psi (\load^{(t)}) \rangle$. The claim now follows by gradient stability (Fact~\ref{fact:PsiSmoothness}) because $0 \leq (C^{(t)} x^{(t)})_j \leq 1$ for all $j$. 
\end{proof}

Applying this bound repeatedly, we can bound the final value of $\Psi$ on the sequence $x^{(1)}, \ldots, x^{(T)}$ in terms of the cost incurred by the algorithm in the surrogate problem as
\begin{equation}
\textstyle
    \Psi(\load^{(T)}) -  \Psi(\load^{(0)})  ~~\leq~~   (1 + \epsilon) \cdot { \sum_{t = 1}^T} \langle C^{(t)} x^{(t)}, \nabla \Psi (\load^{(t-1)}) \rangle.
\end{equation}

As a final remark, we give a bound on the costs of actions in the surrogate problem.
\begin{claim}  \label{claim:LipschitznessOfGame}
For all $t \in [T]$ and all $i \in [n]$, we have $0 \leq c_i^{(t)} \leq \| \one \|_p $. 
\end{claim}
\begin{proof}
By H\"older's inequality, we have for all $t \in [T]$ and all $i \in [n]$ that $c^{(t)}_i = \langle C^{(t)} e_i, \nabla \Psi (\load^{(t-1)}) \rangle$ is at most $\| C^{(t)} e_i \|_p  \cdot \| \nabla \Psi (\load^{(t-1)}) \|_q$, where $q = (1 - \frac{1}{p})^{-1}$. From Fact~\ref{fact:PsiNorm},  $\| \nabla \Psi (\load^{(t-1)}) \|_q \leq 1$. Furthermore, $0 \leq C^{(t)}_{i, j} \leq 1$ for all $i$ and $j$, resulting in $\| C^{(t)} e_i \|_p \leq \| \one \|_p$.
\end{proof}



\section{Load Balancing} \label{sec:loadBalancing}

Recall that in the load balancing problem there is a sequence of cost matrices $C^{(1)}, \ldots, C^{(T)} \in [0,1]^{d\times n}$, which are either drawn i.i.d. from an unknown distribution or defined by an adversary. Our algorithm chooses $x^{(1)}, \ldots, x^{(T)}$  to minimize $\| \load^{(T)} \|_p$, where $\load^{(t)} = \sum_{s = 1}^t C^{(s)} \cdot x^{(s)}$. The point of comparison is any fixed action benchmark $x^\ast \in \Delta_n$, and $\OPTOLVC$ is defined in Eq.~\eqref{eq:defnOPT}.


To solve the load balancing problem, we can apply any no-regret algorithm that is able to cope with an adaptive adversary in the surrogate problem. This ensures that in the surrogate game in hindsight it would not have been much better to always have chosen a different action. Due to Claim~\ref{claim:LipschitznessOfGame}, we have $0 \leq \frac{1}{\| \one \|_p} c_i^{(t)} \leq 1$. As a consequence, $\sum_{t = 1}^T c^{(t)} x^{(t)} \leq \min_{x \in \Delta_n} \sum_{t = 1}^T c^{(t)} x + \| \one \|_p \cdot \Regret$, where $\Regret$ denotes the regret bound for the respective one-dimensional problem in which the costs are bounded between $0$ and $1$. So, for any no-regret algorithm, $\Regret = o(T)$. 

Note that depending on the feedback, different online learning algorithms can be applied. If we get to know $C^{(t)}$ entirely, we can compute the entire vector $c^{(t)}$ and pass it on to the algorithm, which correspond to experts feedback. In this case, Hedge~\cite{FS-GEB99} guarantees $\Regret = O(\sqrt{T \log n})$. If we only get to know $C^{(t)} x^{(t)}$, we can still compute $c^{(t)} x^{(t)}$, which corresponds to bandits feedback. Now, Exp3.P~\cite{AuerCFS-SICOMP02} guarantees $\Regret = \tilde{O}(\sqrt{T \cdot n})$ with high probability.

Combining this regret guarantee with Claim~\ref{claim:PsiSmoothness}, we can bound the final value of $\Psi$ in terms of the cost of any benchmark solution $x^\ast$ in the surrogate problem. In particular, letting $x^\ast$ denote the choice of actions $x^\ast$ that minimizes $\lVert \sum_{t=1}^T C^{(t)} x^\ast \rVert_p$, we have
\begin{align} \Psi(\load^{(T)}) & \leq \Psi(\load^{(0)}) +  ( 1+  \epsilon)\cdot {\textstyle \sum_{t = 1}^T} \langle C^{(t)} x^{(t)}, \nabla \Psi (\load^{(t-1)}) \rangle  \notag \\
& \leq  ( 1+  \epsilon)\cdot {\textstyle \sum_{t = 1}^T}  \langle C^{(t)} x^\ast, \nabla \Psi (\load^{(t-1)}) \rangle  + \| \one \|_p \cdot \Regret   + \Psi(\load^{(0)}). \label{eq:PsiRegret}
\end{align}

The difficulty is to relate the cost $\sum_{t=1}^T  \langle C^{(t)} x^\ast, \nabla \Psi (\load^{(t-1)}) \rangle$ incurred by $x^\ast$ in the surrogate game  to the $\ell_p$ norm of the load it would have generated. Importantly, we cannot apply Claim~\ref{claim:PsiSmoothness} for such a comparison because the costs in the surrogate game are still defined via loads $\load^{(0)}, \ldots, \load^{(T-1)}$ generated by the algorithm and not by $x^\ast$.

In the stochastic setting, we can apply independence between steps. This lets us derive that $ \E [ \| \load^{(T)} \|_p ] \leq ( 1+  \epsilon)\cdot \OPTOLVC +   \| \one \|_p \cdot \Regret  + 2 \frac{p}{\epsilon} \cdot  \big(\|\one\|_p -1 \big)$ in \S\ref{sec:StochLoadBalancing}. In the adversarial setting we have to make much more explicit use of the definition of $\Psi$, which lets us derive $ \| \load^{(T)} \|_p = O( \min\{p, \log d\} \cdot \OPTOLVC + \| \one \|_p \cdot \Regret)$ in \S\ref{sec:AdvLoadBalancing}.


\subsection{Stochastic Arrivals} \label{sec:StochLoadBalancing}

We first consider the stochastic setting where in each time step the cost matrix $C^{(t)} \in [0,1]^{d\times n}$  is sampled i.i.d. from some unknown distribution. For any benchmark $x^\ast$, we define $\OPTOLVC :=  \E \big[ \big\|\sum_{t = 1}^T C^{(t)} \cdot x^\ast \big\|_p \big]$. We derive the following bound on the $\ell_p$ norm of the load.

\begin{theorem}\label{prop:stochLoad}
For stochastic load balancing, any no-regret algorithm applied to the surrogate problem guarantees a load bound of $\E [ \| \load^{(T)} \|_p ] \leq ( 1+  \epsilon)\cdot\OPTOLVC +   \| \one \|_p \cdot \Regret  + 2 \frac{p}{\epsilon} \cdot  \big(\|\one\|_p -1 \big)$, where $\Regret$ is the regret of one-dimensional online learning.
\end{theorem}

For the special case of $\ell_\infty$ norm, we get a bound of $\E [ \| \load^{(T)} \|_\infty ] \leq ( 1+  \epsilon)\cdot\OPTOLVC + \Regret  + 2 \frac{\ln d}{\epsilon}$ because $p(\|\one\|_p -1) \to \ln d$ as $p \to \infty$. 
At the heart of  Theorem~\ref{prop:stochLoad} is the following claim (which is inspired from  \cite{Molinaro-SODA17}) relating the surrogate cost of $x^\ast$  to $\OPTOLVC$.

\begin{claim} \label{claim:StocOPTBound}
If $C^{(1)}, \ldots, C^{(T)}$ are independent, identically distributed random variables, 
\[ 
\E \Big[ \langle C^{(t)} x^\ast , \nabla \Psi (\load^{(t-1)}) \rangle  \Big] ~~\leq ~~ \Big\| \E_{C^{(t)}}[ C^{(t)} x^\ast] \Big\|_p ~~\leq ~~ \frac{\OPTOLVC}{T}.
\]
\end{claim}

\begin{proof}
We first take the expectation only over $C^{(t)}$. As $\nabla \Psi (\load^{(t-1)})$ only depends on the past actions, this yields for the conditional expectation that
\[
\E_{C^{(t)}}[ \langle C^{(t)} x^\ast , \nabla \Psi (\load^{(t-1)}) \rangle  ] =  \langle \E_{C^{(t)}}[C^{(t)} x^\ast] , \nabla \Psi (\load^{(t-1)}) \rangle.
\]
Here, $\E_{C^{(t)}}[C^{(t)} x^\ast]$ denotes the component-wise conditional expectation. As $C^{(1)}, \ldots, C^{(T)}$ are independent, this conditional expectation is equal to the unconditional one. That is,  $\E_{C^{(t)}}[C^{(t)} x^\ast] = \E[C^{(t)} x^\ast]$.  

Applying H\"older's inequality and using $\|\nabla \Psi (\load^{(t-1)})\|_q \leq 1$, we get for any $\load^{(t-1)}$ that
\[\langle \E[ C^{(t)} x^\ast] , \nabla \Psi (\load^{(t-1)}) \rangle ~~\leq~~  \| \E[ C^{(t)} x^\ast] \|_p  \cdot \|\nabla \Psi (\load^{(t-1)})\|_q ~~\leq~~ \| \E[ C^{(t)} x^\ast] \|_p. 
\]
So, in combination
$\E_{C^{(t)}}[ \langle C^{(t)} x^\ast , \nabla \Psi (\load^{(t-1)}) \rangle  ] \leq \| \E[ C^{(t)} x^\ast] \|_p $. 
The right-hand side is not a random variable. So taking  expectation over $C^{(1)}, \ldots, C^{(t-1)}$ proves the first inequality. 

For the second inequality, since $\|\cdot \|_p$ is a convex function it satisfies $\E[\|\cdot \|_p] \geq \|\E[\cdot] \|_p$. Hence, 
\begin{align*}
     \OPTOLVC ~~ =~~  \E \Big[ \Big\|\sum_{t = 1}^T C^{(t)} \cdot x^\ast \Big\|_p \Big]  ~~ \geq ~~  \Big\| \E \Big[ \sum_{t = 1}^T C^{(t)} \cdot x^\ast \Big]  \Big\|_p ~~ = ~~   T \cdot \Big\| \E \big[ C^{(t)} \cdot x^\ast \big]  \Big\|_p,
\end{align*} 
which completes the proof of  Claim~\ref{claim:StocOPTBound}.
\end{proof}

Now the proof of Theorem~\ref{prop:stochLoad} is straightforward.

\begin{proof}[Proof of Theorem~\ref{prop:stochLoad}]
Combining  Claim~\ref{claim:StocOPTBound} with Eq.~\eqref{eq:PsiRegret} and taking expectations,
\begin{align*} \E[\Psi^{(T)}] &\leq   ( 1+  \epsilon)\cdot {\textstyle \sum_{t\leq T} } \E[ \langle C^{(t)} x^\ast, \nabla \Psi (\load^{(t-1)}) \rangle ] + \| \one \|_p \cdot \Regret   + \Psi^{(0)} \\
&\leq ( 1+  \epsilon)\cdot T \cdot  \frac{\OPTOLVC}{T} +    \| \one \|_p \cdot \Regret + \Psi^{(0)} \\
&=  ( 1+  \epsilon)\cdot\OPTOLVC +   \| \one \|_p \cdot \Regret  + \frac{p}{\epsilon} \cdot  \big(\|\one\|_p -1 \big).
\end{align*}
 The theorem now follows by the additive approximation property of $\Psi$ (Fact~\ref{fact:PsiToLoad}).
\end{proof}


\subsection{Adversarial Arrivals}\label{sec:AdvLoadBalancing}

Now we consider the adversarial setting where the cost matrices $C^{(1)}, \ldots, C^{(T)} \in [0,1]^{d\times n}$ are defined by an adversary. Analogous to the stochastic setting, for any benchmark $x^\ast$, we define $\OPTOLVC := \big\|\sum_{t = 1}^T C^{(t)} \cdot x^\ast \big\|_p$. The following is our main result for  adversarial load balancing. 

\begin{theorem}\label{prop:advLoad}
For adversarial load balancing \OLVCp, any no-regret algorithm applied to the surrogate problem guarantees a load bound of $\| \load^{(T)} \|_p \leq  O\big( \min\{p, \log d\} \cdot \OPTOLVC \big) + O\big( \| \one \|_p \cdot \Regret \big)$, where $\Regret$ is the regret of the one-dimensional online learning problem.
\end{theorem}

Like in the stochastic setting, we have to bound the loss of $x^\ast$ in the surrogate game. Claim~\ref{claim:StocOPTBound} uses the stochastic assumptions very heavily and does not hold in this more general case. As a replacement, we show the following lemma. It disentangles the loss of $x^\ast$ in the surrogate game into a factor depending on the actual load that $x^\ast$ generates and the cost that the algorithm incurs.

\begin{lemma}
\label{lemma:AdversarialLossOfOpt}
For any action vector $x^\ast \in \Delta_n$ and $\OPTOLVC := \big\lVert \sum_{t = 1}^T C^{(t)} x^\ast \big\rVert_p$, we have
\[
\sum_{t = 1}^T \langle C^{(t)} x^\ast , \nabla \Psi (\load^{(t-1)}) \rangle  \leq  \Big( {\exp\Big(\frac{\eps  \OPTOLVC}{\| \one \|_p }  \Big) } -1 \Big)  \Big( (1+\eps)   \sum_{t = 1}^T \big\langle C^{(t)} x^{(t)} , \nabla \Psi (\load^{(t-1)}) \big\rangle + \frac{\| \one \|_p}{\eps} \Big).
\]
\end{lemma}


Given Lemma~\ref{lemma:AdversarialLossOfOpt}, the remaining proof of Theorem~\ref{prop:advLoad} is rather straightforward.

\begin{proof}[Proof of Theorem~\ref{prop:advLoad}]
The one-dimensional regret bound in combination with Lemma~\ref{lemma:AdversarialLossOfOpt} gives
\begin{align*}
{\textstyle  \sum_{t = 1}^T } \langle C^{(t)} x^{(t)} , \nabla \Psi (\load^{(t-1)}) \rangle & \leq {\textstyle  \sum_{t = 1}^T } \langle C^{(t)} x^\ast , \nabla \Psi (\load^{(t-1)}) \rangle + \| \one \|_p \cdot \Regret \\
& \hspace{-4cm} \textstyle  \leq  \left( \exp\Big(\frac{\eps  \OPTOLVC}{\| \one \|_p }  \Big) -1 \right) \cdot \left( (1+\eps)  \Big( \sum_{t = 1}^T \langle C^{(t)} x^\ast , \nabla \Psi (\load^{(t-1)}) \rangle \Big) + \frac{\| \one \|_p}{\eps} \right) + \| \one \|_p \cdot \Regret .
\end{align*}
Setting $\epsilon = \min\left\{1, \frac{\| \one \|_p}{5 \cdot \OPTOLVC}\right\}$ gives $\Big( \exp\big(\frac{\eps  \OPTOLVC}{\| \one \|_p}  \big) -1 \Big) \leq \exp( \frac{1}{5} ) - 1 \leq \frac{1}{4}$. Rearranging,
\[ \textstyle
 \Big( 1 - \frac{1+\epsilon}{4} \Big) \sum_{t = 1}^T \langle C^{(t)} x^{(t)} , \nabla \Psi (\load^{(t-1)}) \rangle \quad \leq \quad \frac{\| \one \|_p}{4 \epsilon} + \| \one \|_p \cdot  \Regret,
\]
which implies
$
\sum_{t = 1}^T \langle C^{(t)} x^{(t)} , \nabla \Psi (\load^{(t-1)}) \rangle \leq \frac{\| \one \|_p}{2 \epsilon} + 2 \| \one \|_p \cdot \Regret.
$

Finally, by Claim~\ref{claim:PsiSmoothness}, we have
$
\Psi^{(T)} \leq \Psi^{(0)} + (1 + \epsilon) \sum_{t = 1}^T \langle C^{(t)} x^{(t)} , \nabla \Psi (\load^{(t-1)}) \rangle.
$
Now using Fact~\ref{fact:PsiToLoad} and $\Psi^{(0)} = \frac{p}{\eps} \big( \|\one\|_p -1 \big)$ gives
\begin{align*}
\| \load^{(T)}  \|_p ~~ \leq ~~  \Psi^{(T)} ~~ 
& \leq ~~  \frac{p}{\eps} \big( \|\one\|_p -1 \big) + (1+\epsilon) \frac{\| \one \|_p}{2 \epsilon} + 2 (1+\epsilon) \| \one \|_p \cdot \Regret \\
& \leq 5 \left( 1+ p \frac{\|\one\|_p -1}{\|\one\|_p} \right) \OPTOLVC + 4 \|\one\|_p \Regret + p(\|\one\|_p -1) + \|\one\|_p\\
&= ~~  O\big( \min\{p, \log d\} \cdot \OPTOLVC \big) + O\big( \| \one \|_p \cdot \Regret \big), 
\end{align*}
which proves Theorem~\ref{prop:advLoad}. 
\end{proof}

In the rest of the section we prove the missing Lemma~\ref{lemma:AdversarialLossOfOpt}. We first rewrite $\Psi$ in a different form, namely for all $\load \in \reals^d_{\geq 0}$, 
\[ \textstyle \Psi(\load)  = \frac{p}{\epsilon}\left( \big(\Phi(\load) \big)^{1/p} - 1 \right) \quad \text{where}\quad \Phi (\load) := \sum_{j=1}^d \left( 1 + \frac{\epsilon}{p} \Lambda_j \right)^p. 
\]
For intuition, note that $p\rightarrow \infty$ implies $\Phi(\load) \rightarrow \sum_j \exp(\epsilon \Lambda_j)$, a commonly used potential. 
So,   
\begin{align}\label{eq:gradient} 
\Big(\nabla \Psi(\load^{(t-1)} \Big)_j ~=~ { \big(1+ \frac{\eps}{p} \Lambda_j^{(t-1)} \big)^{p-1}} \cdot {\big(\Phi(\load^{(t-1)}) \big)^{-1/q}}.
\end{align} 

We will need the following property of the function $\Phi$, which follows from the monotonicity of $\nabla \Phi$ and is a slight variant of  Lemma~3.1 in~\cite{Caragiannis-SODA08}. 

\begin{lemma}(\cite{Caragiannis-SODA08}) \label{lem:CaragiannisExp} For any integer $k \geq 1$, any $\load \in \reals^d_{\geq 0}$, and any $z^{(i)} \in \reals^d_{\geq 0}$ for $i\in [k]$, we have $\sum_{i=1}^k \Big( \Phi(\load  +z^{(i)}) - \Phi( \load ) \Big) \leq    \Phi \Big(\load + \sum_{i=1}^k z^{(i)} \Big) - \Phi(\load )$.
\end{lemma}

\begin{proof} [Proof of Lemma~\ref{lem:CaragiannisExp}]
Since $\Phi$ is twice differentiable with monotone $\nabla \Phi$, for any $\alpha \in \reals^d_{\geq 0}$  we have
 \begin{align*} 
 \Phi(\load  +z^{(i)}) - \Phi( \load ) 
 &\textstyle= \int_{s=0}^1 \big\langle \nabla \Phi(\load + s \cdot z^{(i)})~,~ z^{(i)} \big\rangle ~ds   \\
 & \textstyle\leq  \int_{s=0}^1 \big\langle \nabla \Phi(\load + \alpha+ s \cdot z^{(i)})~,~ z^{(i)} \big\rangle ~ds   ~=~ \Phi(\load  + z^{(i)} + \alpha) - \Phi( \load + \alpha ) .
 \end{align*}
 Now taking $\alpha = \sum_{j< i} z^{(j)}$, we can upper bound 
\[\textstyle \sum_{i=1}^k \Big( \Phi(\load  +z^{(i)}) - \Phi( \load ) \Big) \leq  \sum_{i=1}^k   \Big(  \Phi(\load  + \sum_{j\leq i} z^{(j)}) - \Phi( \load + \sum_{j< i} z^{(j)} )  \Big),
\]
which equals $\Phi \Big(\load + \sum_{i=1}^k z^{(i)} \Big) - \Phi(\load )$ and proves the lemma.
\end{proof}

We now have the ingredients to prove Lemma~\ref{lemma:AdversarialLossOfOpt}.
\begin{proof}[Proof of Lemma~\ref{lemma:AdversarialLossOfOpt}]
Using Eq.~\eqref{eq:gradient} and convexity of $\Phi$, we have
\begin{align*} 
\sum_{t = 1}^{T} \langle C^{(t)} x^\ast , \nabla \Psi (\load^{(t-1)}) \rangle  &  = \sum_{t = 1}^{T} \left\langle C^{(t)} x^\ast , \frac{1}{ \eps \Phi^{1/q}(\load^{(t-1)})} \nabla \Phi(\load^{(t-1)}) \right\rangle \\
& \leq   \sum_{t = 1}^{T} \frac{1}{ \eps \Phi^{1/q}(\load^{(t-1)})} \left( \Phi(\load^{(t-1)} + C^{(t)} x^\ast) - \Phi(\load^{(t-1)}) \right).
\end{align*}
For $i\in [T]$,  define  
\begin{align}  \label{eq:defnAi} 
a_i := \frac{1}{ \eps \Phi^{1/q}(\load^{(i-1)})} - \frac{1}{ \eps \Phi^{1/q}(\load^{(i)})}  \quad \text{ and } \quad  a_T := \frac{1}{ \eps \Phi^{1/q}(\load^{(T)})}.
\end{align}
This implies $\frac{1}{ \eps \Phi^{1/q}(\load^{(t-1)})} = \sum_{i = t}^{T-1} \left( \frac{1}{ \eps \Phi^{1/q}(\load^{(i-1)})} - \frac{1}{ \eps \Phi^{1/q}(\load^{(i)})} \right) + \frac{1}{\eps \Phi^{1/q}(\load^{(T)})} = \sum_{i = t}^T a_i$. So, the above inequality can be rewritten as
\begin{align*}
{\textstyle  \sum_{t = 1}^{T} } \langle C^{(t)} x^\ast , \nabla \Psi (\load^{(t-1)}) \rangle & \leq {\textstyle  \sum_{t = 1}^{T} \sum_{i=t}^T} ~a_i \cdot \left( \Phi(\load^{(t-1)} + C^{(t)} x^\ast) - \Phi(\load^{(t-1)}) \right) \\
&= {\textstyle  \sum_{i = 1}^{T}  a_i \cdot \sum_{t=1}^i } \left( \Phi(\load^{(t-1)} + C^{(t)} x^\ast) - \Phi(\load^{(t-1)}) \right) \\
&\leq {\textstyle  \sum_{i = 1}^{T}  a_i \cdot \sum_{t=1}^i} \left( \Phi(\load^{(i-1)} + C^{(t)} x^\ast) - \Phi(\load^{(i-1)}) \right),
\end{align*}
where the last inequality uses monotonicity of $\nabla \Phi$  in every component. Now  applying Lemma~\ref{lem:CaragiannisExp},
\begin{align} \label{eq:perRound}
\textstyle \sum_{t = 1}^{T} \langle C^{(t)} x^\ast , \nabla \Psi (\load^{(t-1)}) \rangle  \leq \sum_{i = 1}^{T}  a_i \cdot  \left( \Phi \Big(\sum_{t = 1}^{i} C^{(t)} x^\ast + \load^{(i-1)} \Big) - \Phi(\load^{(i-1)})   \right). 
\end{align}
To simplify the right-hand side, we use the definition of $\Phi$ and Minkowski's inequality to get
\begin{align*}
\Phi^{1/p}\Big( {\textstyle  \sum_{t = 1}^{i}} C^{(t)} x^\ast + \load^{(i-1)}\Big) &= \Big\| \one + \frac{\eps}{p} \Big({\textstyle  \sum_{t = 1}^{i}} C^{(t)} x^\ast + \load^{(i-1)} \Big) \Big\|_p \\
&\leq \Big\| \one + \frac{\eps}{p}\load^{(i-1)}\Big\|_p + \Big\| \frac{\epsilon}{p} \cdot {\textstyle  \sum_{t = 1}^{i}} C^{(t)} x^\ast  \Big\|_p.
\end{align*}
To further simplify, note that
\[
\Big\| \frac{\epsilon}{p} \cdot \sum_{t = 1}^{i} C^{(t)} x^\ast  \Big\|_p ~~ \leq ~~  \frac{\epsilon}{p} \OPTOLVC ~~ \leq ~~  \frac{\epsilon}{p} \OPTOLVC \cdot \frac{\Phi^{1/p}(\load^{(i-1)})}{d^{1/p}},
\]
where the last step uses that $\Phi(\load^{(i-1)}) \geq d$. Since $\big\| \one + \frac{\eps}{p}\load^{(i-1)}\big\|_p = \Phi^{1/p}(\load^{(i-1)})$, this overall simplifies the above expression to
\begin{align*}
\Phi \Big(\sum_{t = 1}^{i} C^{(t)} x^\ast + \load^{(i-1)}\Big) & \leq \Big(\Phi(\load^{(i-1)})^{1/p} + \frac{\epsilon \OPTOLVC}{p d^{1/p}} \Phi(\load^{(i-1)})^{1/p} \Big)^p \\
& = \Phi(\load^{(i-1)}) \Big(1 + \frac{\epsilon \OPTOLVC}{p d^{1/p}} \Big)^p \quad \leq \quad \Phi(\load^{(i-1)}) \cdot \exp \Big(\frac{\epsilon \OPTOLVC}{d^{1/p}} \Big).
\end{align*}
So, in combination with Eq.~\eqref{eq:perRound}, we get
\begin{align*}
{\textstyle \sum_{t = 1}^{T} \langle C^{(t)} x^\ast , \nabla \Psi (\load^{(t-1)}) \rangle  ~~ \leq ~~  \left(  \exp \big(\frac{\epsilon \OPTOLVC}{d^{1/p}} \big) - 1  \right) \cdot \sum_{i = 1}^{T}  a_i \cdot  \Phi(\load^{(i-1)}) , }
\end{align*}
which  finishes the proof of Lemma~\ref{lemma:AdversarialLossOfOpt} by applying the following Claim~\ref{claim:AiPhiToPsi}.
\end{proof}

\begin{claim} \label{claim:AiPhiToPsi} For $a_i$'s defined in Eq.~\eqref{eq:defnAi}, we have
\[   \sum_{i = 1}^{T}  a_i \cdot  \Phi(\load^{(i-1)})   ~~\leq ~~ (1+\eps)  \Big( \sum_{t = 1}^T \langle C^{(t)} x^{(t)} , \nabla \Psi (\load^{(t-1)}) \rangle \Big) + \frac{d^{1/p}}{ \eps }.
\]
\end{claim}
\begin{proof} We start by expanding the left-hand side to get
\begin{align*}
\textstyle \sum_{i = 1}^{T}  a_i \cdot  \Phi(\load^{(i-1)})  
& \textstyle = \sum_{i=1}^T \Big( \frac{1}{ \eps \Phi^{1/q}(\load^{(i-1)})} - \frac{1}{ \eps \Phi^{1/q}(\load^{(i)})}  \Big) \cdot  \Phi(\load^{(i-1)})\\
& \textstyle = \sum_{i=2}^T \Big( \frac{\Phi(\load^{(i-1)}) - \Phi(\load^{(i-2)})}{ \eps \Phi^{1/q}(\load^{(i-1)})}  \Big) + \frac{\Phi(\load^{(0)})}{ \eps \Phi^{1/q}(\load^{(0)})} - \frac{\Phi(\load^{(T-1)})}{ \eps \Phi^{1/q}(\load^{(T)})} \\
& \textstyle \leq  \sum_{i=2}^T \Big( \frac{\Phi(\load^{(i-2)} + C^{(i-1)}x^{(i-1)}) - \Phi(\load^{(i-2)})}{ \eps \Phi^{1/q}(\load^{(i-1)})}  \Big) + \frac{d^{1/p}}{ \eps } .
\end{align*}
Since $\Phi(\cdot)$ is a convex function, 
$ \Phi(\load^{(i-2)} + C^{(i-1)}x^{(i-1)}) - \Phi(\load^{(i-2)}) \leq   \langle \nabla \Phi(\load^{(i-1)}), C^{(i-1)}x^{(i-1)}) \rangle.
$
Moreover, recall from Eq.~\eqref{eq:gradient} that $\nabla \Psi^{(i-1)} = \frac{\nabla \Phi(\load^{(i-1)})}{\eps \Phi^{1/q}(\load^{(i-1)})}$. Thus, we get
\begin{align*} 
\textstyle \sum_{i = 1}^{T}  a_i \cdot  \Phi(\load^{(i-1)})  & \textstyle \leq  \sum_{i=2}^T   \langle \nabla \Psi^{(i-1)}, C^{(i-1)}x^{(i-1)} \rangle   + \frac{d^{1/p}}{ \eps } \\
&\textstyle \leq (1+\eps)  \sum_{i=2}^T   \langle \nabla \Psi^{(i-2)}, C^{(i-1)}x^{(i-1)} \rangle   + \frac{d^{1/p}}{ \eps } \\
& \textstyle  = (1+\eps)  \Big( \sum_{t = 1}^{T-1} \langle C^{(t)} x^{(t)} , \nabla \Psi (\load^{(t-1)}) \rangle \Big) + \frac{d^{1/p}}{ \eps },
\end{align*}
where the last inequality uses the stability of gradient of $\Psi$ (Fact~\ref{fact:PsiSmoothness}).
\end{proof}


\section{Bandits with Knapsacks}

In this section we consider \BwKp as defined in \S\ref{sec:introBwK}.  The goal is to maximize the total collected reward while ensuring that the $\ell_p$ norm of the total cost vector stays within the budget $B$. Recall, for $\ell_{\infty}$ norm this captures in the stochastic case  the setting of~\cite{BKS-JACM18} and in the adversarial case  the setting of~\cite{ISSS-FOCS19}. 
The high-level plan of our proofs is to use the ideas  from \S\ref{sec:loadBalancing} to reduce the $d$-dimensional budget constraint to a single-dimensional budget constraint, and then to Lagrangify this single budget constraint.

\subsection{Adversarial Bandits with Knapsacks} \label{sec:BwKAdversarial}

 In this section we present how to execute this plan in the adversarial case, where the proof is easier as we can afford losing a constant multiplicative factor in the approximation. In \S\ref{sec:BwKAdversarial} we discuss the stochastic case where we need more work and use a different norm.


To deal with the multi-dimensional budget constraint, we follow the approach described in \S\ref{sec:surrogateGame} and define $\Psi(\load) = \frac{p}{\epsilon} \big( \| \one + \frac{\epsilon \load}{p} \|_p - 1 \big)$. We again consider a surrogate one-dimensional online learning game with $n$ actions. 
This time, the algorithm's goal is to maximize the following suitably defined reward function. In time step $t$, the \emph{Lagrangian reward} of the $i$-{th} action/expert is set to
\[
\mathcal{R}^{(t)}_i = \begin{cases} r^{(t)} e_i - \lambda 	\cdot \langle C^{(t)} e_i, \nabla \Psi (\load^{(t-1)}) \rangle & \text{ if $\lVert \load^{(t-1)} \rVert_p \leq B$} \\
0 & \text{ otherwise,}
\end{cases}
\]
where $r^{(t)}$ is the reward vector from \BwKp problem,  $\load^{(t)} = \sum_{s = 1}^t C^{(s)} \cdot x^{(s)}$ is the load after time~$t$, and $\nabla \Psi(\load^{(t-1)})$ is the gradient of $\Psi$. We will use 
\[ 
\lambda :=  \frac{\sum_{t = 1}^{\tau^\ast} r^{(t)} \cdot x^\ast}{2B} =  \frac{\OPTBwK}{2 B} \qquad \text{ and } \qquad  \epsilon := \frac{2 p (\|\one\|_p - 1)}{B} , \]
where $\tau^\ast$ is the time at which $x^\ast$ runs out of budget, otherwise $\tau^\ast = T$. 
Note that  the rewards are always zero after the total cost/load vector exceeds budget $B$. Moreover, the algorithm can implement this game online since we are assuming the adversary is adaptive (non-oblivious).

\begin{claim} \label{claim:BwLAdvBoundedExpert}
The reward of each expert per time step is bounded  in magnitude by $(\lambda\cdot \| \one_d \|_p +1)$.
\end{claim}
\begin{proof}By H\"older's inequality, we have $\langle C^{(t)} e_i, \nabla \Psi(\load^{(t-1)}) \rangle \leq \|C^{(t)} e_i  \|_p \cdot \|\nabla \Psi(\load^{(t-1)})\|_q \leq \| \one \|_p$, where the last step uses $\|\nabla \Psi(\load^{(t-1)})\|_{q}  \leq 1$ by Fact~\ref{fact:PsiNorm} and  $\|C^{(t)} e_i  \|_{p} \leq  \| \one \|_p$.

Consequently $\lvert \mathcal{R}^{(t)} e_i \rvert \leq \lvert r^{(t)} e_i \rvert + \lvert \lambda \langle C^{(t)} e_i, \nabla \Psi(\load^{(t-1)}) \rangle \rvert \leq \lambda\cdot \| \one \|_p +1$.
\end{proof}

Let $\Regret_{t}$ denote the regret of a $1$-dimensional online learning algorithm against an adaptive adversary at time $t \in \{1,\ldots, T\}$.

Recall, $x^\ast$ is defined to be a benchmark solution satisfying $\| \sum_{t=1}^T C^{(t)} x^\ast \|_p \leq B$. Our point of comparison will be a scaled-down version $x' := \frac{1}{\alpha}  x^\ast$, where $\alpha := 5 p \frac{\| \one \|_p - 1}{\| \one \|_p}$ (note $\alpha \rightarrow 5\ln d$ as $p \rightarrow \infty$). 
The intuition for scaling  $x^\ast$ is that for adversarial \OLVC we only know from Theorem~\ref{prop:advLoad} that the budget is not exceeded multiplicatively by more than a factor of $p$. So to \emph{strictly} satisfy the budget constraints in \BwKp, we need to scale-down the solution. 

The following is our main result for adversarial \BwKp.

\begin{theorem} \label{prop:BwKAdversarial}
If $B \geq 2 p (\|\one\|_p-1)$ then for adversarial \BwKp the algorithm gets a reward
\[ \textstyle
\sum_{t = 1}^\tau r^{(t)} \cdot x^{(t)} ~~\geq~~ \frac{1}{20\min\{p, \ln d\}}\OPTBwK- O\Big(\frac{\OPTBwK \cdot \| \one \|_p }{B}\Big)\cdot \Regret.
\]
\end{theorem}
\begin{proof}
We distinguish the analysis in two cases: Either the algorithm  stops before  time $\tau^\ast$ (recall, this is time at which  $x^\ast$ runs out of budget, otherwise $\tau^\ast = T$) or it stays within budget until $\tau^\ast$.

\paragraph{Case 1: Algorithm stops before $\tau^\ast$, i.e., $\lVert \load^{(\tau^\ast)} \rVert_p > B$.}

In this case, there is some time $\tau \leq \tau^\ast$ at which $\lVert \load^{(\tau)} \rVert_p > B$ for the first time. This means that $\tau$ is the last round before the algorithm stops and $\mathcal{R}^{(t)} = 0$ for all $t > \tau$. Since Claim~\ref{claim:BwLAdvBoundedExpert} bounds the loss of each action/expert, we get $O\big(\lambda\cdot \| \one_d \|_p \big) \cdot\Regret_{\tau^\ast}$ bounds the Lagrangian reward of the algorithm compared to always playing the null action, which has Lagrangian reward $0$. This gives 
\[ \textstyle
\sum_{t = 1}^{\tau} \mathcal{R}^{(t)} \cdot x^{(t)} ~~=~~ \sum_{t = 1}^{\tau^\ast} \mathcal{R}^{(t)} \cdot x^{(t)} ~~\geq~~ 0 - O\big(\lambda\cdot \| \one_d \|_p \big)\cdot  \Regret_{\tau^\ast}.
\]
Using the definition of $\mathcal{R}^{(t)}$, this implies
\begin{align*} \textstyle
\sum_{t = 1}^{\tau} r^{(t)} \cdot x^{(t)} ~~\geq~~ \lambda \sum_{t = 1}^{\tau} \langle C^{(t)} x^{(t)} , \nabla \Psi(\load^{(t-1)}) \rangle - O\big(\lambda\cdot \| \one_d \|_p \big)\cdot \Regret_{\tau^\ast}
\end{align*}
We have $\Psi(\load^{(\tau)}) - \Psi(0) \geq \| \load^{(\tau)} \|_p - \frac{p}{\epsilon} \cdot \big(\|\one\|_p -1 \big) = \frac{B}{2}$. So, by Claim~\ref{claim:PsiSmoothness},
\[ \textstyle
\sum_{t = 1}^{\tau} \langle C^{(t)} x^{(t)} , \nabla \Psi(\load^{(t-1)}) \rangle ~~\geq~~ \frac{1}{1 + \epsilon} (\Psi(\load^{(\tau)}) - \Psi(0)) ~~\geq~~ \frac{B}{2 (1+\epsilon)}.
\]
This implies
\begin{align*} \textstyle
\sum_{t = 1}^{\tau} r^{(t)} \cdot x^{(t)} ~ \geq~  \frac{\lambda \cdot B}{2(1 + \epsilon)} - O\big(\lambda\cdot \| \one_d \|_p \big)\cdot \Regret_{\tau^\ast} ~= ~ \frac{\OPTBwK}{4 (1 + \epsilon)} - O\big(\lambda\cdot \| \one_d \|_p \big)\cdot \Regret_{\tau^\ast}.
\end{align*}

\paragraph{Case 2: Algorithm stays within budget till $\tau^\ast$, i.e., $\lVert \load^{(\tau^\ast)} \rVert_p \leq B$.}

In this case, $\lVert \load^{(t)} \rVert_p \leq B$ for all relevant $t$. This time, we use that $\Regret_{\tau^\ast}$ bounds the Lagrangian reward compared to always playing $x'$. That is,
$
\sum_{t = 1}^{\tau^\ast} \mathcal{R}^{(t)} \cdot x^{(t)} \geq \sum_{t = 1}^{\tau^\ast} \mathcal{R}^{(t)} \cdot x' - \Regret_{\tau^\ast},
$
and therefore by the definition of $\mathcal{R}^{(t)}$,
\begin{align} \label{eq:BwKAdvers}
\sum_{t = 1}^{\tau^\ast} r^{(t)} \cdot x^{(t)} \geq \sum_{t = 1}^{\tau^\ast} \left( r^{(t)} \cdot x' - \lambda \cdot \langle C^{(t)} x' , \nabla \Psi(\load^{(t-1)}) \rangle + \lambda \cdot \langle C^{(t)} x^{(t)} , \nabla \Psi(\load^{(t-1)}) \rangle \right) \notag\\
\textstyle - ~O\big(\lambda\cdot \| \one_d \|_p \big)\cdot \Regret_{\tau^\ast}.
\end{align}
Now apply Lemma~\ref{lemma:AdversarialLossOfOpt} to $x'$, noting that $\lVert \sum_{t = 1}^{\tau^\ast} C^{(t)} x' \rVert_p \leq \frac{B}{\alpha}$. So,
\begin{align*} 
\sum_{t = 1}^{\tau^\ast} \langle C^{(t)} x' , \nabla \Psi(\load^{(t-1)}) \rangle 
& \leq  \Big( \exp\big(\frac{\eps}{\| \one \|_p } \frac{B}{\alpha} \big) -1 \Big)  \Big( (1+\eps)  \sum_{t = 1}^{\tau^\ast} \langle C^{(t)}  x^{(t)} , \nabla \Psi(\load^{(t-1)}) \rangle  + \frac{\| \one \|_p}{ \eps } \Big) \\
& = \big( e^{2/5} -1 \big)  \Big( (1+\eps)  \sum_{t = 1}^{\tau^\ast} \langle C^{(t)}  x^{(t)} , \nabla \Psi(\load^{(t-1)}) \rangle  + \frac{\| \one \|_p}{ \eps } \Big) \\
& \leq  \sum_{t = 1}^{\tau^\ast} \langle C^{(t)}  x^{(t)} , \nabla \Psi(\load^{(t-1)}) \rangle  + \frac{\| \one \|_p}{ 2 \eps } ,
\end{align*}
where the last step uses that $\epsilon \leq 1$. We can rewrite this inequality as
\begin{align*} 
- \lambda \sum_{t = 1}^{\tau^\ast} \langle C^{(t)} x' , \nabla \Psi(\load^{(t-1)}) \rangle &  + \lambda \sum_{t = 1}^{\tau^\ast} \langle C^{(t)} x^{(t)} , \nabla \Psi(\load^{(t-1)}) \rangle  ~~\geq~~ - \lambda \frac{\| \one \|_p}{2 \eps} \\
& \qquad\qquad =   -\lambda \frac{B}{4 p} \frac{\| \one \|_p}{\| \one \|_p - 1}~ \geq -  \frac{1}{8 p} \frac{\| \one \|_p}{\| \one \|_p - 1} \sum_{t = 1}^{\tau^\ast} r^{(t)} \cdot x^\ast
\end{align*}
by our choice of $\lambda = \frac{\sum_{t = 1}^{\tau^\ast} r^{(t)} \cdot x^\ast}{2B} $ and $\epsilon = \frac{2 p (\|\one\|_p - 1)}{B}$.
Substituting this in Eq.~\eqref{eq:BwKAdvers} gives
\begin{align*}
 \sum_{t = 1}^{\tau^\ast} r^{(t)} \cdot x^{(t)} &  \geq  \sum_{t = 1}^{\tau^\ast} \frac{1}{\alpha} \langle  r^{(t)}, x^\ast \rangle  - \frac{1}{8 p} \frac{\| \one \|_p}{\| \one \|_p - 1} \sum_{t = 1}^{\tau^\ast} r^{(t)} \cdot x^\ast - O\big(\lambda\cdot \| \one_d \|_p \big)\cdot \Regret_{\tau^\ast} \\
&  \geq \frac{1}{20 \min\{p, \ln d\}}\sum_{t = 1}^{\tau^\ast} r^{(t)} \cdot x^\ast  - O\big(\lambda\cdot \| \one_d \|_p \big)\cdot \Regret_{\tau^\ast}
\end{align*}
because $\alpha = 5p \frac{\| \one \|_p - 1}{\| \one \|_p}$, which completes the proof.
\end{proof}

\subsection{Stochastic Bandits with Knapsacks} \label{sec:BwKAdversarial}
In this section we consider the case where the input is sampled i.i.d. from an unknown distribution. We want to still Lagrangify the budget constraint with a parameter $\lambda$, but there is a new challenge that we cannot afford losing a multiplicative factor of $2$ since we want a $1+o(1)$ approximation.  The analysis for adversarial arrivals in \S\ref{sec:BwKAdversarial} considers two cases, and to balance between them we set $\lambda = \frac{\OPTBwK}{2B}$, but  any such analysis can  at best only give us $\OPTBwK/2$ reward.

For simplicity, let's first address the above challenge in the $\ell_\infty$ case. Here the idea, which first appeared in~\cite{BKS-JACM18}, is to introduce a dummy-resource of ``time'' by adding another dimension.  Every action (including the null action) incurs a loss of $\frac{B}{T}$ in each time step in this new dimension, say dimension $0$. Note that since we are working with the $\ell_\infty$ norm, adding this new dimension does not change $\OPTBwK$ since  playing $x^\ast$ is still feasible. The benefit of introducing this dummy resource is that now our algorithm is guaranteed to exhaust its budget $B$ by time $T$. At the same time, we can assume that $x^\ast$ does not exhaust its budget before time $T$ because we can scale it down. This implies we are always in Case~1 of the adversarial analysis, and can therefore set $\lambda \approx\frac{\OPTBwK}{B}$.

Although the idea of adding a dummy resource works in the $\ell_\infty$ case, it's not clear how to implement it for general $\ell_p$ norms. This is because the dummy resource will contribute to the overall $\ell_p$ load,  unlike in the $\ell_\infty$ case, which means playing $x^\ast$ is no longer feasible and the structure of the problem significantly changes. We still manage to obtain the following sublinear regret.

\begin{theorem} \label{prop:BwKStochastic}
For stochastic \BwKp, there exists an algorithm that given $\OPTBwK$ gets a reward
\[ 
\E \left[\sum_{t = 1}^\tau r^{(t)} \cdot x^{(t)} \right] \geq \OPTBwK ~ -~ \OPTBwK \left(  \left( \frac{ \| \one_d \|_p }{B} \right)^{1/3} + \left(\frac{ p\cdot \| \one_d \|_p }{B} \right)^{1/2}+ \frac{ \| \one_d \|_p }{B}  \Regret \right),
\]
where $\Regret$ is the regret of one-dimensional online learning.
\end{theorem}

For the important special case of $\ell_\infty$ norm, we approximate $\ell_\infty$ by choosing $p= \frac{\ln d}{\gamma}$. This gives a bound of $\E \left[\sum_{t = 1}^\tau r^{(t)} \cdot x^{(t)} \right] \geq \OPTBwK (1-4\gamma)$ whenever $B \geq \max\{ \frac{\ln d  \cdot\| \one_d\|_p }{\gamma^{3}},   \frac{\| \one_d \|_p \cdot  \Regret}{\gamma}\}$.

To prove Theorem~\ref{prop:BwKStochastic}  (it will follow from Proposition~\ref{prop:BwKStochasticProp}), we have to define a different norm.

\paragraph{A Different Norm.} To overcome the above challenge, we work with a different   $\|\cdot\|_{p,r}$ norm. For any $(d+1)$-dimensional vector $\mathbf{x}= (x_0,\ldots, x_d)$, we define
\[ \| \mathbf{x}\|_{p,r} := \Big\| \Big(x_0 ~,~ \|(x_1,\ldots,x_d) \|_p \Big) \Big\|_r.
\]
The benefit of this norm is that for $r=\infty$ it allows us to introduce a dummy resource in dimension $0$ without affecting  the optimum strategy $x^\ast$. 

Although powerful, the $\|\cdot\|_{p,r}$ norm introduces new challenges: we need  an approximation $\Psi_{p,r}(\cdot)$ whose gradient is stable, akin to \S\ref{sec:surrogateGame}.  Fortunately, this is possible by defining 
\[	\Psi_{p,r}(\mathbf{x}) :=   \frac{1}{\delta} \left( \Big( 1 +  \delta x_0 \Big)^r  + \Big( \big\|\one_d + \delta\mathbf{y}\big\|_p \Big)^r  \right)^{1/r} - \frac{1}{\delta},
\]
where $\delta = \frac{\eps}{p+r}$. We will need the following properties of $\Psi_{p,r}(\cdot)$.

\begin{fact} \label{fact:PsiNewToLoad} (Additive Approximation)
For any integer $p,r\geq 1$ and load $\load \in \reals^d_{\geq 0}$, we have $\| \load  \|_{p,r} \leq  \Psi_{p,r}(\load)  \leq  \| \load \|_{p,r} + \frac{p+r}{\epsilon} \cdot  \|\one_d\|_p$.
\end{fact}

\begin{restatable}{fact}{PsiNewSmoothness}\label{fact:PsiNewSmoothness} (Gradient Stability)
For any integers $p, r \geq 1$, load $\load \in \reals^{d+1}_{\geq 0}$, and load increase $z \in \reals^{d+1}_{\geq 0}$ with $0 \leq z_j \leq 1$ for all $j$, coordinate-wise $\nabla_{p,r} \Psi(\load + z) \leq \big(1 + O(\epsilon)\big) \cdot \nabla \Psi_{p,r}(\load)$.
\end{restatable}

\begin{proof}[Proof of Fact~\ref{fact:PsiNewSmoothness}] 
Denote $\mathbf{y} := (x_1, \ldots, x_d)$.  Recall, $\delta:= \frac{\eps}{p+r}$ and
\begin{align*}
\Psi_{p,r}(x_0,\mathbf{y}) := \frac{1}{\delta} \left( \Big( 1 +  \delta x_0 \Big)^r  + \Big( \big\|\one_d + \delta\mathbf{y}\big\|_p \Big)^r  \right)^{1/r} - \frac{1}{\delta}.
\end{align*}
For ease of notation, denote 
\[\Phi_{p,r}(x_0,\mathbf{y}) := \Big( 1 + \delta x_0 \Big)^r  + \Big( \big\|\one_d + \delta \mathbf{y}\big\|_p \Big)^r \quad \text{and} \quad \Phi_{p}(\mathbf{y}) := \Big(\big\|\one_d + \delta \mathbf{y}\big\|_p\Big)^p.
\]
Now we can write the gradient of $\Psi$ as
\[ \frac{\partial \Psi_{p,r}(x_0,\mathbf{y}) }{\partial x_0} = \frac{1}{\Phi_{p,r}^{1-1/r}} \Big( 1 + \delta  x_0 \Big)^{r-1} 
\quad \text{and} \quad \frac{\partial \Psi_{p,r}(x_0,\mathbf{y}) }{\partial y_i} =\frac{\Phi_{p}^{r/p -1 }}{\Phi_{p,r}^{1-1/r}} \Big( 1 +  \delta y_i \Big)^{p-1} . 
\]

We note that if we increase each coordinate of $(x_0,\mathbf{y})$ by at most $1$, then
\begin{itemize} 
\item the term in the denominators $\Phi_{p,r}^{1-1/r}$ only increases.  
\item the terms $\Big( 1 + \delta  x_0 \Big)^{r-1} $ and $\Big( 1 +  \delta y_i \Big)^{p-1} $   increase by a factor at most $\exp(\eps)$ since $\delta \leq \min\{\frac{1}{p-1},\frac{1}{r-1}\}$.
\item finally, the term $\Phi_{p}^{r/p -1}$ decreases if $r\leq p$, and otherwise when $r>p$ it increases at most by a factor of $\exp(\delta p  \cdot \frac{r}{p}) = 1+O(\eps)$. 
\end{itemize}
Thus we have that the gradient increases by at most a $1+O(\eps)$ factor coordinate-wise.
\end{proof} 

\begin{fact}\label{fact:PsiNewNorm} (Gradient Norm) For any integers $p, r \geq 1$, we have $\|\nabla \Psi_{p,r}\|_{q,s} \leq 1$ where $q=(1-1/p)^{-1}$ and $1/s=(1-1/r)^{-1}$.
\end{fact}

We  now use $\Psi_{p,r}$ and its properties  to define and analyze a surrogate problem.

\paragraph{Surrogate Online Learning Problem.}
We play a surrogate one-dimensional online learning game with $n$ experts. Besides consuming the $d$ resources, every action (including null) also consumes $\frac{B}{T}$ of resource $0$. 
The load $\load^{(t)} = \sum_{s = 1}^t C^{(s)} \cdot x^{(s)}$ is a $(d+1)$ dimensional vector where the first row of $C^{(s)}$ (i.e., dimension $0$) contains all $\frac{B}{T}$. In time step $t$ the reward of the $i$-th expert is 
\[
\mathcal{R}^{(t)}_i = \begin{cases} r^{(t)} e_i - \lambda 	\cdot \langle C^{(t)} e_i, g^{(t)} \rangle & \text{ if $\lVert \load^{(t-1)} \rVert_{p,\infty} \leq B$} \\
0 & \text{ otherwise,}
\end{cases}
\]
where $r^{(t)}$ is the reward vector from \BwKp problem and $g^{(t)} := \nabla \Psi_{p,r}(\load^{(t-1)})$.  
We will  set
\begin{align}\label{eq:defLambdaEps}
\textstyle \lambda :=  \frac{\sum_{t = 1}^{\tau^\ast} r^{(t)} \cdot x^\ast}{B} =  \frac{\OPTBwK}{ B} \qquad \text{ and } \qquad  \epsilon :=  \sqrt{\frac{(p+r) \cdot  \|\one_d\|_p }{B }} . 
\end{align}

\begin{claim} The reward of each expert per time step is bounded  in magnitude by $(\lambda\cdot \| \one_d \|_p +1)$.
\end{claim}
\begin{proof} Since $|r^{(t)} e_i|\leq 1$, we argue $ | \langle C^{(t)} e_i, g^{(t)} \rangle | \leq \| \one \|_p $. By generalized Cauchy-Schwarz,
\[ \langle C^{(t)} e_i, g^{(t)} \rangle \leq \|C^{(t)} e_i  \|_{p,r} \cdot \|g^{(t)}\|_{q,s} .
\]
This completes the proof because $\|g^{(t)}\|_{q,s}  \leq 1$ by Fact~\ref{fact:PsiNewNorm} and  $\|C^{(t)} e_i  \|_{p,r} \leq  \| \one_d \|_p$.
\end{proof}

Let $x^\ast$  be a benchmark solution with the property that $\| \E[ C^{(t)} x^\ast ] \|_{p,r} \leq \frac{B}{T}$. The following proposition proves Theorem~\ref{prop:BwKStochastic} for $r= \Theta \Big( \big( \frac{B}{ \| \one_d \|_p } \big)^{1/3} \Big)$.


\begin{proposition} \label{prop:BwKStochasticProp}
For stochastic \BwKp and any $r\geq 1$, the algorithm gets a reward
\[ 
\E \left[\sum_{t = 1}^\tau r^{(t)} \cdot x^{(t)} \right] \geq \OPTBwK  \left(  1  - \left({\frac{(p+r) \cdot  \|\one_d\|_p }{B }}\right)^{1/2} - \big(2^{1/r} - 1 \big)  ~ - \frac{ \| \one_d \|_p }{B} \cdot \Regret \right).
\]
\end{proposition}

\begin{proof} The dummy resource is exhausted after time $T$, so the algorithm runs out of some resource before or at time $T$. 
Let $\tau$ be the random time at which this happens.  For analysis, we assume our algorithm stops collecting reward at time $\tau$. 

We use that $\Regret$ bounds $\sum_{t=1}^T \mathcal{R}^{(t)} \cdot x^{(t)}$ compared to the option of always playing $x^\ast$. By definition $\mathcal{R}^{(t)} = 0$ for $t > \tau$. So, this gives us
\begin{align} \label{eq:stochBwKRegretBound}
\textstyle \sum_{t=1}^\tau \mathcal{R}^{(t)} \cdot x^{(t)} \geq \sum_{t=1}^\tau \mathcal{R}^{(t)} \cdot x^\ast - O\big(\lambda\cdot \| \one_d \|_p \big)\cdot \Regret.
\end{align}


Combining $\Psi^{(\tau)}_{p,r} \geq \| \load^{(\tau)} \|_{p, \infty} \geq B$ with Claim~\ref{claim:PsiSmoothness}, which holds because of gradient stability (Fact~\ref{fact:PsiNewSmoothness}), and using $(1+\eps)^{-1}\geq (1-\eps)$ gives
\[ \textstyle \sum_{t\leq \tau} \langle C^{(t)} x^{(t)}, g^{(t)} \rangle \quad \geq \quad \left( 1 - \epsilon \right) \cdot \left( B - \Psi^{(0)}_{p,r} \right) .
\]
This along with Eq.~\eqref{eq:stochBwKRegretBound} implies
\begin{align} 
\sum_{t=1}^\tau r^{(t)} \cdot x^{(t)} & \geq \lambda \sum_{t=1}^\tau \langle C^{(t)} x^{(t)}, g^{(t)} \rangle + \sum_{t=1}^\tau \mathcal{R}^{(t)} \cdot x^\ast - O\big(\lambda\cdot \| \one_d \|_p \big)\cdot \Regret \notag \\
& \geq \lambda \left( 1 - \epsilon \right) \cdot \left(  B  - \Psi^{(0)}_{p,r} \right) + \sum_{t=1}^\tau \mathcal{R}^{(t)} \cdot x^\ast - O\big(\lambda\cdot \| \one_d \|_p \big)\cdot \Regret .
\label{eq:stochBwKrtxt}
\end{align}

Now, let us consider $x^\ast$. For $t \in [T]$, define a random variable $Z_t$ with $Z_t = 1$ if $\lVert \load^{(t-1)} \rVert_{p,\infty} \leq B$ and $0$ otherwise. This lets us rewrite $\mathcal{R}^{(t)} \cdot x^\ast$ as
\[
\mathcal{R}^{(t)} \cdot x^\ast = Z_t \left( r^{(t)} x^\ast - \lambda \cdot \big\langle C^{(t)} x^\ast, g^{(t)} \big\rangle \right).
\]
Now, we take the expectation over $r^{(t)}$ and $C^{(t)}$, which are by definition independent of $Z_t$ because $Z_t$ only depends on steps $1, \ldots, t-1$. This gives us
\[
\E_{r^{(t)}, C^{(t)}} [\mathcal{R}^{(t)} \cdot x^\ast] = Z_t \left( \E_{r^{(t)}, C^{(t)}}[ r^{(t)} x^\ast] - \lambda \cdot \big\langle \E_{r^{(t)}, C^{(t)}}[C^{(t)} x^\ast]~,~ g^{(t)} \big\rangle \right).
\]
We have $\E_{r^{(t)}, C^{(t)}}[ r^{(t)} x^\ast] = \frac{\OPTBwK}{T}$ and using H\"older's inequality 
\[\Big\langle \E_{r^{(t)}, C^{(t)}}[C^{(t)} x^\ast], g^{(t)} \Big\rangle ~~\leq~~ \Big\| \E_{r^{(t)}, C^{(t)}}[C^{(t)} x^\ast] \Big\|_{p,r} \cdot \|g^{(t)}\|_{q,s} ~~\leq~~ 2^{1/r} \frac{B}{T}.
\]

So, in combination after taking the expectation over rounds $1, \ldots, t$, we have $\E[\mathcal{R}^{(t)} \cdot x^\ast] \geq \E[Z_t] \frac{1}{T} \OPTBwK - \lambda \frac{1}{T} B$ and by linearity of expectation
 \[
 \E \Big[\sum_{t=1}^T \mathcal{R}^{(t)} \cdot x^\ast \Big] \geq \frac{\E[\tau]}{T} \cdot \OPTBwK - \lambda \cdot 2^{1/r} \cdot \frac{\E[\tau]}{T} B.
\]
Combining this with Eq.~\eqref{eq:stochBwKrtxt} and by linearity of expectation,
\begin{align*}
&\E \Big[\sum_{t=1}^\tau r^{(t)} \cdot x^{(t)} \Big]  \geq \lambda \left( 1 - \epsilon \right) \cdot \left(  B  - \Psi^{(0)}_{p,r} \right) +\E \Big[ \sum_{t=1}^\tau \mathcal{R}^{(t)} \cdot x^\ast \Big] - O\big(\lambda\cdot \| \one_d \|_p \big)\cdot \Regret\\
& \geq \lambda \left( 1 - \epsilon \right) \cdot \left(  B  - \Psi^{(0)}_{p,r} \right) + \frac{\E[\tau]}{T} \OPTBwK - \lambda 2^{1/r} \frac{\E[\tau]}{T} B - O\big(\lambda\cdot \| \one_d \|_p \big)\cdot \Regret \\
& \geq  \OPTBwK \cdot \left( 1 - \epsilon \right) \cdot \left(  1  - \frac{(p+r) \cdot  \|\one_d\|_p  }{\epsilon B}  \right) - (2^{1/r}-1) \OPTBwK - O\big(\lambda\cdot \| \one_d \|_p \big)\cdot \Regret,
\end{align*}
where the last inequality uses $\lambda = \frac{\OPTBwK}{B}$ and $\Psi^{(0)}_{p,r} \leq \frac{p+r}{\epsilon} \cdot  \|\one_d\|_p$. Finally, use  $\eps$ from Eq.~\eqref{eq:defLambdaEps}. 
\end{proof}

\medskip
\noindent
{\bf Acknowledgments}.
We are thankful to the authors of \cite{ISSS-FOCS19} for explaining us their results. We thank the anonymous reviewers of COLT 2020  for helpful comments on improving the presentation of the paper.

{\small
\bibliographystyle{alpha}
\bibliography{bib}
}

\appendix



\section{Lower Bounds for Adversarial Arrivals} \label{sec:lowerBounds}
In this section we prove our adversarial arrivals lower bounds, both for \OLVCp and \BwKp. These lower bounds hold even if we assume the algorithm knows the value of $\OPTOLVC$ or $\OPTBwK$.


\subsection{Load Balancing}

\begin{proposition}
Any online algorithm for adversarial \OLVCp is $\Omega(\min\{p, \log d\})$ competitive.
\end{proposition}
\begin{proof} Our lower bound instance is inspired from the \OGLBp lower bound instance of~\cite{AwerbuchAGKKV-FOCS95}. For simplicity, assume $d$ is a power of $2$. 
There are $k = \min\{p, \log_2 d\}$ phases of requests, where the length  of each phase is $L=\frac{T}{k}$.  In the first phase, we think of all the $d$ dimensions as being \emph{active}. At the end of Phase $i \in [k]$, an unbiased random coin $R_i \in \{0,1\}$ makes either all the top half (for $R_i=0$) or the bottom half (for $R_i=1$) of the currently active dimensions inactive. This means  in Phase $i$ there are exactly $\frac{d}{2^{i-1}}$ active dimensions.
There are $2^k$ actions where we think of each action $a$ as a $k$ bit string. 
In the $i$-th phase, pulling action $a$ puts a unit load on all the top or bottom half of active dimensions, depending on whether the $i$-th bit of $a$ is $0$ or $1$, respectively.

We first observe that for the above instance, the action $a = (R_1,R_2,\ldots, R_p)$  puts exactly $L$ load in every dimension, which means the $\ell_p$ norm  of its total load is $L \cdot \|\one\|_p$. This is true this action puts unit load on exactly those dimensions that will become inactive in the next phase. Thus each arm gets non-zero load in at most one phase.

Next we argue that the $\ell_p$ norm of the load vector of any online learning algorithm is $\Omega(L \cdot  \|\one\|_p \cdot k)$. This is because an online algorithm does not know $R_i$ until the beginning of Phase~$i+1$. Thus at best it can evenly distribute between the actions with the $i$-th bit being $0/1$.  Hence its smallest possible $\ell_p$ norm is given when $\frac{d}{2}$ dimensions have $\frac{L}{2}$ load, $\frac{d}{2^2}$ dimensions have $\frac{2L}{2}$ load, \ldots, $\frac{d}{2^i}$ dimensions have $\frac{i L}{2}$ load. In particular, each of the $d/2^k$ dimensions that is active throughout the sequence has a load of $\frac{L k}{2}$. This lower-bounds the $\ell_p$ norm by
$\left( \frac{d}{2^k} \big(\frac{L k}{2} \big)^p \right)^{1/p} = \frac{1}{2^{1 + k/p}} d^{1/p} L k = \Theta\big( L \cdot \|\one\|_p \cdot k \big)$.
\end{proof}


\subsection{Bandits with Knapsacks}

\begin{proposition}
Any online algorithm for adversarial \BwKp is $\Omega(\min\{p, \log d\})$ competitive.
\end{proposition}
\begin{proof}
The idea is to repeat the same instance as the lower bound instance for adversarial \OGLBp with budget $L\cdot \|\one\|_p$. Moreover, each non-null action gives a unit-reward. We know from the analysis that there exists an optimal action that satisfies budget for $k$ phases, resulting in a reward of $k L$. Next we argue that every online algorithm playing non-null actions satisfies  the budget for at most $O(1)$ phases.

By repeating the analysis for \OLVCp, we know the best an online algorithm can do is to distribute the load uniformly among the active dimensions. This means if the final reward is $z$ then all $d/2^k$ dimensions which are active in all phases must have aggregated a load of $\frac{z}{2}$ each. This gives a lower-bound for the $\ell_p$ norm of $\big(\frac{d}{2^k} \big(\frac{z}{2} \big)^p \big)^{1/p}$. For feasibility this has to be bounded by budget, i.e., $\big(\frac{d}{2^k} \big(\frac{z}{2} \big)^p \big)^{1/p} \leq L\cdot \|\one\|_p = L d^{1/p}$, which implies that $z \leq 2^{1 + k/p} L = \Theta(L)$.
\end{proof}



\section{Removing the Assumption that $\OPT$ is Known}

For  adversarial $\OLVCp$, it is possible to possible to prove our results from \S\ref{sec:loadBalancing} without the assumption that the value of $\OPTOLVC$ is known by only losing a small constant factor. For adversarial \BwKp, however, one has to lose an additional $\Omega(\log T)$ factor when $\OPTBwK$ is not known. One cannot avoid this $\Omega(\log T)$ factor as was shown in~\cite{ISSS-FOCS19}. 

\paragraph{Load Balancing}
 For adversarial \OLVCp, we use the standard doubling trick to handle that the algorithm doesn't know the optimal load $\OPTOLVC$. The algorithm operates in phases, where in Phase $i \geq 1$ it guesses the $\ell_p$ norm of the optimal solution to be $2^i$. Under this assumption, we run the adversarial \OLVCp from \S\ref{sec:loadBalancing} where in the beginning of each phase the algorithm imagines that it starts with $0$ load in each of the dimensions. 
If ever during the execution of the algorithm, the load of the algorithm becomes more than $c\cdot p \cdot 2^i$, where $c$ is some sufficiently large constant, the algorithm knows that  $\OPTOLVC$ must have been higher than $2^i$ as otherwise it violates the $O(p)$-competitive guarantee from \S\ref{sec:loadBalancing}. Thus the algorithm moves to Phase~$i+1$, where its guess for  $\OPTOLVC$ is $2^{i+1}$. 

Now we argue that the above algorithm is still $O(p)$ competitive. This happens because by Minkowski's inequality, we can upper bound the $\ell_p$ norm of the algorithm's total load as the sum of its $\ell_p$ load in each  phase. Moreover, notice the algorithm never goes past Phase $i^\ast$ where $\OPTOLVC \in [2^{i^\ast -1}, 2^{i^\ast})$. Hence the total load of the algorithm is at most
$\sum_{i=1}^{i^\ast} c\cdot p \cdot 2^i = O(p \cdot 2^{i^\ast}) = O(p \cdot \OPTOLVC)$.


\paragraph{Bandits with Knapsacks}
  For adversarial \BwKp, one can get an $O(\min\{p, \log d\} \log T)$  competitive ratio by first guessing the value of $\OPTBwK$ up to a factor of $2$. Note that $\OPTBwK \in [1,T]$, so by exponential bucking $[1,2), [2,2^2), \ldots, [2^i,, 2^{i+1})$, the algorithm can achieve this with probability at least $\frac{1}{\log T}$. Thus the expected reward of the algorithm is $O(\min\{p, \log d\} \log T)$-competitive  against the fixed optimal action distribution $x^\ast$.


\end{document}